\documentclass[a4paper]{ifacconf}

\usepackage{graphicx}

\makeatletter
\let\old@ssect\@ssect 
\makeatother

\usepackage{natbib}
\usepackage{hyperref}
\hypersetup{
    colorlinks=true,
    linkcolor=black,
    filecolor=black,
    urlcolor=blue,
    citecolor=blue,
    }
\makeatletter
\def\@ssect#1#2#3#4#5#6{%
  \NR@gettitle{#6}
  \old@ssect{#1}{#2}{#3}{#4}{#5}{#6}
}
\makeatother

\usepackage{bbm}
\usepackage{bm}

\usepackage[algo2e]{algorithm2e}

\makeatletter
\DeclareRobustCommand{\qed}{%
  \ifmmode 
  \else \leavevmode\unskip\penalty9999 \hbox{}\nobreak\hfill
  \fi
  \quad\hbox{\qedsymbol}}
\newcommand{\openbox}{\leavevmode
  \hbox to.77778em{%
  \hfil\vrule
  \vbox to.675em{\hrule width.6em\vfil\hrule}%
  \vrule\hfil}}
\newcommand{\qedsymbol}{\openbox}
\newenvironment{proof}[1][\proofname]{\par
  \normalfont
  \topsep6\p@\@plus6\p@ \trivlist
  \item[\hskip\labelsep\itshape
    #1.]\ignorespaces
}{%
  \qed\endtrivlist
}
\newcommand{\proofname}{Proof}
\makeatother

\input{preamble/math}

  \newcommand{\newcolorlabel}[2]{%
  \expandafter\newcommand\csname #1\endcsname[1]{%
    \tikz[baseline]{\node[text=white,fill=#2,anchor=base,text height=1.3ex,text depth=0.1ex,font=\sffamily\bfseries]{##1}}}%
}

\newcommand{\newcommenter}[2]{%
  \expandafter\newcommand\csname #1\endcsname[1]{%
    \fcolorbox{#2}{#2}{\color{white}\textsf{\textbf{#1}}}
    {\color{#2}##1}}%
  \expandafter\newcommand\csname at#1\endcsname{%
    \fcolorbox{#2}{#2}{\color{white}\textsf{\textbf{@#1}}}
    {\color{#2}}}%
  \expandafter\newcommand\csname #1cite\endcsname[1]{%
    \csname #1\endcsname {[##1]}
  }%
  \expandafter\newcommand\csname #1ref\endcsname[1]{%
    \csname #1\endcsname {$\blacktriangleright$##1}
  }%
  \expandafter\newcommand\csname #1hl\endcsname[2]{%
    \colorbox{#2}{\color{white}\textsf{\textbf{#1}}}\sethlcolor{Azure2}\hl{##2}~%
    \expandafter\ifx\csname commentarrow\endcsname\relax$\leftarrow$\else \commentarrow[#2]\fi~%
    {\color{#2}##1}}%
  \expandafter\newcommand\csname #1st\endcsname[2]{%
    \colorbox{#2}{\color{white}\textsf{\textbf{#1}}}\sout{##2}~%
    \expandafter\ifx\csname commentarrow\endcsname\relax$\leftarrow$\else \commentarrow[#2]\fi~%
    {\color{#2}##1}}%
}
\newcommenter{TODO}{DodgerBlue1}
\newcommenter{rtron}{green!80!black}

\usepackage{comment}

\usepackage{pdfcomment}

\usepackage{soul}

\usepackage[normalem]{ulem}

\usepackage{csquotes}

\DeclarePairedDelimiter{\norm}{\lVert}{\rVert}

\begin{document}

\begin{frontmatter}

\title{
A Graph-Based Approach to Generate Energy-Optimal Robot Trajectories in Polygonal Environments
}

\author[Logan]{Logan E. Beaver}
\author[Roberto]{Roberto Tron}
\author[Christos]{Christos G. Cassandras}

\address[Logan]{Division of Systems Engineering, Boston University,\\ Boston, MA 02215 USA (e-mail: lebeaver@bu.edu).}
\address[Roberto]{Department of Mechanical Engineering
at Boston University, Boston, MA, 02215 USA (e-mail: tron@bu.edu).}
\address[Christos]{Division of Systems Engineering and
Department of Electrical and Computer Engineering, Boston University,
Boston, MA 02215 USA (e-mail: cgc@bu.edu).}

\thanks[footnoteinfo]{This work was supported in part by NSF under grants ECCS-1931600, DMS-1664644, CNS-1645681, CNS-2149511, by AFOSR under grant FA9550-19-1-0158, by ARPA-E under grant DE-AR0001282, by the MathWorks and by NPRP grant (12S-0228-190177) from the Qatar National Research Fund, a member of the Qatar Foundation (the statements made herein are solely the responsibility of the authors).}

\begin{abstract}

As robotic systems continue to address emerging issues in areas such as logistics, mobility, manufacturing, and disaster response, it is increasingly important to rapidly generate safe and energy-efficient trajectories.
In this article, we present a new approach to plan energy-optimal trajectories through cluttered environments containing polygonal obstacles.
In particular, we develop a method to quickly generate optimal trajectories for a double-integrator system, and we show that optimal path planning reduces to an integer program.
To find an efficient solution, we present a distance-informed prefix search to efficiently generate optimal trajectories for a large class of environments.
We demonstrate that our approach, while matching the performance of RRT* and Probabilistic Road Maps in terms of path length, outperforms both in terms of energy cost and computational time by up to an order of magnitude.
We also demonstrate that our approach yields implementable trajectories in an experiment with a Crazyflie quadrotor.
\end{abstract}


\end{frontmatter}

\section{Introduction}

Robotic systems are exploding in popularity, and in particular the potential for mobile robots to address emerging challenges in logistics, e-commerce, and transportation has skyrocketed. 
At the same time, there has been a push in the robotics and control systems communities to focus on long-duration autonomy, where robotic systems are left to perform tasks in an environment for much longer than a typical laboratory experiment, \cite{Egerstedt2021RobotAutonomy,Notomista2019Constraint-DrivenSystems}.
Managing the energy consumption of these robotic platforms is critical for the large-scale deployment of robotic systems, and implementing energy-optimal control strategies can significantly reduce the energy requirements of individual robots, \cite{chalaki2021CSM}.
This allows robots to remain in the field longer, \cite{Notomista2019TheRobot}, and reduces their battery size--which has a strong correlation with cost and environmental impact.

Generating energy-optimal trajectories to navigate complex environments is of critical importance for these applications.
One popular approach is to apply Control Barrier Functions (CBFs) to guarantee safety while driving the robot toward a destination using a reference input, \cite{Xiao2021BridgingVehicles}.
While these methods perform well and are robust, \cite{Bahreinian2021RobustProgramming}, they are fundamentally local, and there is a significant risk of the robot becoming stuck in a local minima without additional information about the environment. 
A practical technique to generate these higher-level trajectories is collocation, i.e., finding optimal parameters for a set of basis functions.
Under this approach, a designer selects an appropriate basis for their application, e.g., polynomial splines, \cite{Mellinger2011MinimumQuadrotors,Sreenath2013TrajectorySystem}, and the parameters of these basis functions are determined to yield an approximately optimal trajectory.
However, the performance of these approaches is limited by the choice of basis.

Sampling algorithms are a compelling alternative to collocation techniques, as many asymptotically approach the globally optimal solution as the number of samples increases.
Optimal Rapidly Exploring Random Trees (RRT*) and Probabilistic Road Maps (PRM) are two techniques that are frequently employed for complex environments, \cite{Karaman2010IncrementalPlanning,Geraerts2004APlanners}.
In both cases, states are randomly sampled from the domain and connected together to form a path between the initial and final state.
As the name implies, RRT* generates a tree that branches out from the initial (or final) state to generate a trajectory, whereas PRM samples the environment to find a feasible sequence of states.
Both approaches interpolate between states to generate a trajectory, but neither captures the energy consumption of the entire trajectory. 
%
In contrast, we solve the optimality conditions directly to generate a path for a mobile robot.
Despite the difficulties associated with optimal control (see \cite{Bryson1996Optimal1985}), recent work by \cite{Aubin-Frankowski2021LinearlyMethods} has applied kernel methods to quickly generate optimal trajectories for systems with linear dynamics, and a recent result by \cite{Beaver2021DifferentialTime} outlines an approach to quickly generate optimal trajectories for differentially flat systems, such as quadrotors, \cite{Morrell2018DifferentialFlight,Faessler2018DifferentialTrajectories}.

In this work, we build on \cite{Beaver2021DifferentialTime} to generate energy-optimal trajectories for a robot with double integrator dynamics to navigate an environment filled with polygonal obstacles.
We transform the optimal control problem into an equivalent integer program, where the decision variables are the sequence of constraint activations.
We generate an approximate solution using a distance-informed search over the graph of possible constraints, and we prove that our algorithm's complexity has polynomial scaling in the dimension of the problem. 
We also demonstrate in simulation that our approach outperforms RRT* and PRM in terms of energy cost and computational time while matching the performance of RRT* in terms of path length.

The remainder of this article is organized as follows: we formulate an energy-optimal trajectory problem and list our assumptions in Section \ref{sec:formulation}.
In Section \ref{sec:solution}, we generate the solution to the optimal control problem, and we transform the solution into an equivalent graph problem in Section \ref{sec:algorithm}.
We present simulation-based and laboratory experimental results in Sections \ref{sec:sim} and \ref{sec:experiment}, and we draw conclusions and present some directions for future work in Section \ref{sec:conclusion}.

\section{Problem Formulation} \label{sec:formulation}

Consider a robot with double-integrator dynamics operating in $\mathbbm{R}^{2}$,
\begin{equation}
\begin{aligned} \label{eq:dynamics}
    \dot{\bm{p}} &= \bm{v}, \\
    \dot{\bm{v}} &= \bm{u},
\end{aligned}
\end{equation}
where $\bm{p}, \bm{v}\in\mathbb{R}^2$ are the position and velocity of the robot, and $\bm{u}\in\mathbb{R}^2$ is the control input; we denote the state of the robot by $\bm{x} = [\bm{p}^T, \bm{v}^T ]^T$.
Our objective is to generate an energy-minimizing trajectory from an initial state $\bm{x}^0$ at time $t^0$ to a final state $\bm{x}^f$ at time $t^f$ while avoiding obstacles in the environment.
We approximate the robot's energy consumption with the $L^2$ norm of the control input, which is a standard energy surrogate in the literature on connected and automated vehicles, \cite{chalaki2021CSM}, ecologically-inspired robotics , \cite{Notomista2019Constraint-DrivenSystems}, and quadrotors, \cite{Mellinger2011MinimumQuadrotors}.
This leads to the cost functional,
\begin{equation} \label{eq:cost}
  J(\bm{u}) = \frac{1}{2}\int_{t^0}^{t^f} \norm{\bm{u}}^2 dt.
\end{equation}
Finally, the environment contains $N$ polygonal obstacles, where each of the $l \in \{1, 2, \dots, N\}$ polygons consists of $n_l$ vertices.
We index the vertices and faces using the set $\mathcal{F} = \{1, 2, \dots, n_1, \dots, n_1+n_2,\dots,\sum_{l=1}^{N} n_l  \}$ (the number of faces equals the number of vertices).
The obstacles are described by a vertex set and a face function,
\begin{align}
    \mathcal{V} &= \big\{ \bm{c}_k \in \mathbb{R}^2, k \in \mathcal{F} \big\}, \\
  \bm{f} &: k\in\mathcal{F} \mapsto (\bm{c}_{k_1}, \bm{c}_{k_2}) \in (\mathcal{V} \times \mathcal{V}),
\end{align}
where $\mathcal{V}$ defines the embedding of each vertex $\bm{c}_k$ in $\mathbb{R}^2$, and $\bm{f}(k)$ maps each face $k$ to a pair of vertices.

Each edge is a convex combination of its vertices, which we parameterize in terms of the distance $d\in(0, D_k)$,
\begin{equation}
    \bm{s}_k(d) = \frac{D_k - d}{D_k} \bm{c}_{k_1} + \frac{d}{D_k} \bm{c}_{k_2},
\end{equation}
where $D_k = ||\bm{c}_{k_2} - \bm{c}_{k_1}||$ is the distance between the vertices of edge $k$.
Each face has a corresponding tangent vector $\hat{\bm{t}}_k = \frac{\bm{c}_{k_2} - \bm{c}_{k_1}}{D_k}$ and a normal vector $\hat{\bm{n}}_k \perp \hat{\bm{t}}_k$ that is pointing away from the polygon.
Next, to ensure safety, we construct a function to determine the robot's distance from any edge $k\in\mathcal{F}$.
First, the projected position of the robot along $\bm{s}_k$ is,
\begin{equation} \label{eq:projection}
    P_{k}(\bm{p}) = (\bm{p} - \bm{c}_{k_1}) \cdot \hat{\bm{t}}_k,
\end{equation}
and then the robot's distance from the polygon edge is defined as
\begin{equation} \label{eq:signedDistance}
    d_k\big(\bm{p}\big) =
    \begin{cases}
        ||\bm{p} - \bm{s}_k \big( P_k(\bm{p}) \big) ||  & \text{if } P_k(\bm{p})\in[0, D_k], \\
        \norm{\bm{p} - \bm{c}_{k_1}} & \text{if } P_k(\bm{p}) < 0, \\
        \norm{\bm{p} - \bm{c}_{k_2}} & \text{if } P_k(\bm{p}) > D_k,\\
    \end{cases}
\end{equation}
which is the robot's distance to the nearest point of face $k$.
We ensure safety by imposing the collision-avoidance constraint,
\begin{equation} \label{eq:constraint}
    g_k(\bm{p}) = R - d_k(\bm{p}) \leq 0
\end{equation}
where $R \in \mathbb{R}_{>0}$ is the radius of a circle circumscribing the robot.
Finally, we formalize our optimal control policy for the robot in Problem \ref{prb:ocp}.
\begin{problem} \label{prb:ocp}
Generate the control trajectory that satisfies,
\begin{align*}
    \min_{\bm{u}} & \int_{t^0}^{t^f} \frac{1}{2} ||\bm{u}||^2 \, dt \\
    \text{subject to: }&  \\
    &\eqref{eq:dynamics}, \,
    \Big[\bm{x}(t^0), \bm{x}(t^f)\Big]  = \Big[ \bm{x}^0, \bm{x}^f \Big], \\
     &R - d_{k}(\bm{p}) \leq 0, \quad \forall k\in\mathcal{F}.
\end{align*}
\end{problem}

In order to solve Problem \ref{prb:ocp}, we introduce the following working assumptions.

\begin{assumption}[Tracking] \label{smp:tracking}
    The robot is capable of tracking the optimal solution to Problem \ref{prb:ocp}.
\end{assumption}

In general, integrator dynamics are useful for quickly planning trajectories, and they are kinematically feasible for linear systems and nonlinear systems that satisfy differential flatness, \cite{Murray1995DifferentialSystems}.
However, noise, disturbances, and unmodeled dynamics may result in tracking errors.
These challenges can be addressed through learning techniques (see \cite{Greeff2021ExploitingProcesses}) or by implementing a low-level tracking controller with control barrier functions (see \cite{Xiao2021BridgingVehicles,Ames2019ControlApplications}).

\begin{assumption}[Control Bounds] \label{smp:bounds}
    The control effort is upper bounded, and the time horizon $t^f - t^0$ is sufficiently large to ensure the control effort remains below the bound.
\end{assumption}

Assumption \ref{smp:bounds} simplifies our analysis to only consider state constraints that affect the trajectory.
This is the current standard for kernel-based approaches to optimal control (see \cite{Aubin-Frankowski2021LinearlyMethods}), and Assumption \ref{smp:bounds} can always be satisfied by increasing the final time $t^f$.
Note that this assumption is not critical for our approach; state and control bound constraints can be explicitly included in Problem \ref{prb:ocp} using the standard approach of \cite{Bryson1975AppliedControl}.

\begin{assumption}[Reasonable Environment] \label{smp:environment}
    We start the robot from rest, the obstacles are spaced at least a distance of $2R$ apart, a feasible path exists, and the environment is known a priori.
\end{assumption}

We employ Assumption \ref{smp:environment} to simplify our analytical solution in the sequel.
Our proposed approach can be expanded to handle any violation of Assumption \ref{smp:environment}; for example, a receding horizon control approach could be used in situations where the environment is not known a priori.
However, we omit these more complex cases for brevity.

\section{Solution Approach} \label{sec:solution}

We solve Problem \ref{prb:ocp} using the standard approach for continuous-time optimal control problems, \cite{Bryson1975AppliedControl}.
The obstacle avoidance constraint in Problem \ref{prb:ocp} is not an explicit function of the control action $\bm{u}$, thus the first step is to take two time derivatives so that the control input appears.
Taking a time derivative of \eqref{eq:constraint} yields,
\begin{align}
    \frac{d}{dt} \big( R - d_k(\bm{p}) \big) &= - \frac{d}{dt} d_k(\bm{p}) = -\frac{d}{dt}\Big(\bm{p} - \bm{s}_k(P_k(\bm{p}))\Big)\cdot\hat{\bm{n}}_k \notag\\
    &= -\bm{v}\cdot\hat{\bm{n}}_k + \frac{\partial \bm{s}_k}{\partial P_k}\frac{\partial P_k}{\partial \bm{p}}\frac{d \bm{p}}{dt} \cdot\hat{\bm{n}}_k \notag\\
    &= - \bm{v}\cdot\hat{\bm{n}}_k + (\bm{v}\cdot\hat{\bm{t}}_k)\hat{\bm{t}}_k\cdot\hat{\bm{n}}_k \notag\\
    &= -\bm{v}\cdot\hat{\bm{n}}_k.
    \label{eq:tangency}
\end{align}
Taking the derivative a second time, we obtain,
\begin{equation}
    \frac{d^2}{dt^2}\big(R - d_k(\bm{p})\big) = -\bm{u}\cdot\hat{\bm{n}}_k.
\end{equation}
Thus, the Hamiltonian is,
\begin{align} \label{eq:hamiltonian}
    H(\bm{x}, \bm{u}, \bm{\lambda}) =& \frac{1}{2}||\bm{u}||^2 + \bm{v}\cdot\bm{\lambda}^p + \bm{u}\cdot\bm{\lambda}^v \notag\\
    &- \sum_{k\in\mathcal{F}} \mu_k\big( \bm{u}\cdot\hat{\bm{n}}_k \big),
\end{align}
where $\bm{\lambda}^p$ and $\bm{\lambda}^v$ are the costates and $\mu_k$ are the inequality Lagrange multipliers that satisfy,
\begin{equation} \label{eq:lagrangeMult}
    \begin{cases}
        \mu_k = 0 \text{ if } \Big(R - d_k(\bm{p}) \Big) < 0, \\
        \mu_k \geq 0 \text{ if } \Big(R - d_k(\bm{p})\Big) = 0.
    \end{cases}
\end{equation}

Our objective is to find the optimal control trajectory that minimizes \eqref{eq:hamiltonian} and yields a solution for Problem \ref{prb:ocp}.
The robot follows integrator dynamics by our premise, thus our problem is amenable to the approach proposed in \cite{Beaver2021DifferentialTime}.
This leads to expressions for
the costates,
\begin{align}
    \bm{\lambda}^p &= \dot{\bm{u}} - \sum_{k\in\mathcal{F}}\dot{\mu}_k\hat{\bm{n}}_k, \label{eq:lambdaP} \\
    \bm{\lambda}^v &= \sum_{k\in\mathcal{F}}\mu_k\hat{\bm{n}} - \bm{u}, \label{eq:lambdaV}
\end{align}
and the control input that minimizes the Hamiltonian \eqref{eq:hamiltonian},
\begin{equation} \label{eq:ode}
    \ddot{\bm{u}} - \sum_{k\in\mathcal{F}}\ddot{\mu}_k\cdot\hat{\bm{n}}_k = 0.
\end{equation}
Equations \eqref{eq:lambdaP}--\eqref{eq:ode} are derived by manipulating the Euler-Lagrange and optimality conditions in \cite{Beaver2021DifferentialTime}.

The control trajectory $\bm{u}$ that solves \eqref{eq:ode} is a piece-wise function of unconstrained (singular) and constrained (non-singular) arcs that must be pieced together at so-called ``junctions'' to recover the optimal trajectory.
Let $t_1$ denote the instant that at least one constraint becomes active or inactive; we refer to this time and (unknown) state as a junction between two arcs.
Each junction has two possibilities:
\begin{enumerate}
    \item the constraint(s) becomes active for a single time-instant, and
    \item the constraint(s) becomes active over a non-zero interval of time.
\end{enumerate}
The first case corresponds to the robot instantaneously touching the face(s) of any polygon obstacle, while the second case corresponds to the robot moving parallel to polygon edge $k$ while satisfying $d_k(\bm{p}) = 0$ over a non-zero time interval.
Both cases may occur simultaneously; for example, if the robot touches vertex $k$ at $t_k$ and remains on one polygon edge over a non-zero interval of time afterward.

While a general solution to the constrained motion problem is possible, under Assumption \ref{smp:environment} we only consider Case 1 where the robot interacts with the polygon vertices instantaneously.
This leads to the following results, which can be extended to cover both cases.

\begin{lemma} \label{lma:feasible}
    Let the robot come into contact with vertex $i$ at time $t_k$, and let the two edges connected to this vertex have outward facing normal vectors $\hat{n}_{k_1}$ and $\hat{n}_{k_2}$.
    If an obstacle is convex at vertex $\bm{c}_i$, then the robot's motion is feasible if and only if,
    \begin{equation} \label{eq:lmaFeasible}
        \bm{v}(t_k)\cdot\hat{\bm{n}}_{k_1} \geq 0 \text{ or } \bm{v}(t_k)\cdot\hat{\bm{n}}_{k_2} \geq 0.
    \end{equation}
\end{lemma}

\begin{proof}
    By our premise, $\bm{p} = \bm{c}_i$ and the robot obeys integrator dynamics.
    This implies,
    \begin{equation} \label{eq:integral}
        (\bm{p}(t_k + \Delta t) - \bm{c}_i) \cdot \hat{\bm{n}}_f = \int_{t_k}^{t_k+\Delta t} \bm{v}(t_k)\cdot\hat{\bm{n}}_f dt,
    \end{equation}
    for each face $f \in \{k_1, k_2\}$ and arbitrarily small $\Delta t$.
    First, we prove the necessity of Lemma \ref{lma:feasible} by a contrapositive argument; let $\bm{v}(t_k)\cdot\hat{\bm{n}}_f < 0$ for each face $f\in\{k_1, k_2\}$.
    Continuity in $\bm{p}$ implies that, for any control input $\bm{u}$, there exists $\Delta t > 0$ such that the right hand side of \eqref{eq:integral} is negative.
    This implies that the robot's position is inside the polygon, and thus \eqref{eq:lmaFeasible} is a necessary condition for feasibility.
    Next, without loss of generality, let $\bm{v}(t_k)\cdot\hat{\bm{n}}_{k_1} \geq 0$.
    Again, continuity in $\bm{p}$ implies that \eqref{eq:integral} is non-negative for $\Delta t > 0$.
    This implies that the robot is outside of the polygon when vertex $\bm{c}_i$ is a convex corner, and thus \eqref{eq:lmaFeasible} is sufficient to guarantee feasibility.
\end{proof}

\begin{corollary} \label{cor:stop}
    If at time $t_k$ the robot contacts vertex $\bm{c}_i$ that is non-convex, the robot's trajectory is feasible if and only if $\bm{v}(t_k) = 0$.
\end{corollary}

\begin{proof}
    When a vertex $\bm{c}_i$ is non-convex, the condition
    \begin{equation} \label{eq:stop}
        \bm{v}(t_k) \cdot \hat{\bm{n}}_{k_1} \geq 0 \text{ and } \bm{v}(t_k) \cdot \hat{\bm{n}}_{k_2} \geq 0,
    \end{equation}
    is necessary and sufficient for feasibility, following the proof of Lemma \ref{lma:feasible}.
    However, for the robot to approach vertex $\bm{c}_i$, the robot's velocity must satisfy
    \begin{equation}
        \bm{v}(t_k) \cdot \hat{\bm{n}}_{k_1} < 0 \text{ and } \bm{v}(t_k) \cdot \hat{\bm{n}}_{k_2} < 0,
    \end{equation}
    for all $t\in[t_k - \Delta t, t_k)$ for sufficiently small $\Delta t > 0$.
    Thus, by continuity in $\bm{v}$, the only feasible solution is
    \begin{equation}
        \bm{v}(t_k) \cdot \hat{\bm{n}}_{k_1} = \bm{v}(t_k) \cdot \hat{\bm{n}}_{k_2} = 0.
    \end{equation}
    The vectors $\hat{\bm{n}}_{k_1}$ and $\hat{\bm{n}}_{k_2}$ span $\mathbbm{R}^2$ by our premise, which implies that $\bm{v}(t_k) = 0$.
\end{proof}

Next, we examine how the obstacle avoidance constraint \eqref{eq:constraint} affects the robot's trajectory when it comes into contact with a vertex.
In this case, the vertex essentially becomes an interior-point constraint that the robot must satisfy along its trajectory.

\begin{lemma} \label{lma:interiorPt}
    Let the robot instantaneously touch vertex $\bm{c}_j\in\mathcal{V}$ at time $t_i$ that joins faces $k_1$ and $k_2$.
    If the resulting trajectory is feasible (Lemma \ref{lma:feasible} or Corrolary \ref{cor:stop}), then the the following conditions are sufficient for optimality,
    \begin{align}
        \bm{p}(t_i) = \bm{c_j}, \bm{v}(t_i^-) = \bm{v}(t_i^+), \bm{u}(t_i^-) = \bm{u}(t_i^+), \notag\\
        \big(\dot{\bm{u}}(t_i^-) - \dot{\bm{u}}(t_i^+)\big) \cdot \bm{v}(t_i) = 0.
    \end{align}
\end{lemma}

\begin{proof}
Both faces $k_1$ and $k_2$ are captured by the constraint,
\begin{equation}
    d_{k_1}(\bm{p}) = d_{k_2}(\bm{p}) = ||\bm{p} - \bm{c}_j|| = R,
\end{equation}
at some unknown time $t_i$.
The resulting tangency condition is \citep{Bryson1975AppliedControl},
\begin{equation}
    N(\bm{x}, t) = 
        ||\bm{p} - \bm{c}_j|| - R,
\end{equation}
and the optimality conditions at $t_k$ are,
\begin{equation} \label{eq:lambdaJump}
    \bm{\lambda}^{T^-} = \bm{\lambda}^{T^+} + \pi \frac{\partial}{\partial \bm{x}} \bm{N},
\end{equation}
where $\pi$ is a constant Lagrange multiplier and the superscripts $^-$ and $^+$ correspond to variable evaluations just before and after the constraint activation time, respectively.
Combining this with \eqref{eq:lambdaV} yields,
\begin{equation}
    \bm{\lambda}^{v^-} - \bm{\lambda}^{v^+} = \bm{u}^- - \bm{u}^+ = 0,
\end{equation}
and thus the control input is continuous.
Furthermore, the Hamiltonian optimality condition is,
\begin{align}
    H^- &= H^+ + \frac{\partial}{\partial t} \bm{N} \notag\\
    &= -\frac{1}{2}||\bm{u}||^- + \bm{\lambda}^{p^-}\cdot\bm{v} = -\frac{1}{2}||\bm{u}||^+ + \bm{\lambda}^{p^+}\cdot\bm{v},
\end{align}
substituting \eqref{eq:lambdaP} and applying continuity of $\bm{u}$ yields the final condition,
\begin{equation}
    \Big(\dot{\bm{u}}^- - \dot{\bm{u}}^+\Big)\cdot\bm{v} = 0.
\end{equation}
\end{proof}

\begin{theorem} \label{thm:optimality}
    The optimal trajectory is a sequence of unconstrained arcs connected with junctions satisfying Lemma \ref{lma:interiorPt}.
\end{theorem}

\begin{proof}
    Under Assumption \ref{smp:environment}, the optimal trajectory of the robot may only touch obstacles instantaneously at a vertex.
    Thus, the control input that solves \eqref{eq:ode} is a piecewise-linear function that satisfies
    \begin{equation}
        \ddot{\bm{u}} = \bm{0}
    \end{equation}
    almost everywhere, except for the constraint junctions.
    Each unconstrained arc has $8$ unknowns, and we prove that the conditions of Lemma \ref{lma:interiorPt} are sufficient to completely determine the trajectory, which is optimal by definition, using induction.

    Let $n$ be the number of unconstrained arcs in the optimal trajectory, and let a trajectory with $n$ segments have an equal number of equations (optimality and boundary conditions) and unknowns (trajectory coefficients and junction times).
    Next, introduce a new junction at time $t_1$ such that one of the unconstrained arcs is split into two pieces; this creates a new system of equations with $n+1$ unconstrained arcs.
    This introduces $9$ additional unknowns ($8$ from the coefficients of the additional unconstrained arc and $1$ from the unknown time $t_1$).
    Lemma \ref{lma:interiorPt} yields the corresponding $9$ equations,
    $\bm{p}(t_1^-) = \bm{c}_i$,
    $\bm{p}(t_1^+) = \bm{c}_i$,
    $\bm{v}(t_1^-) = \bm{v}(t_1^+)$,
    $\bm{u}(t_1^+) = \bm{u}(t_1^+)$,
    and $(\dot{\bm{u}}^+(t_1) + \dot{\bm{u}}^-(t_1))\cdot\bm{v}(t_1) = 0$.
    Finally, consider the case when $n=1$; this is a system of $8$ equations (boundary conditions) and $8$ unknowns (trajectory coefficients).
\end{proof}

In the following section we apply Theorem \ref{thm:optimality} to convert Problem \ref{prb:ocp} into a shortest path problem.

\section{Optimal Trajectory Generation} \label{sec:algorithm}

Theorem \ref{thm:optimality} proves that the optimal trajectory can be constructed from a sequence of linear control arcs of the form,
\begin{equation}
    \begin{aligned} \label{eq:MP}
        \bm{p}_j(t) &= \bm{a}_{3,j}\,t^3  + \bm{a}_{2,j}\,t^2 + \bm{a}_{1,j}\,t + \bm{a}_{0,j}, \\
        \bm{v}_j(t) &= 3\bm{a}_{3,j}\,t^2 + 2\bm{a}_{2,j}\,t  + \bm{a}_{1,j}, \\
        \bm{u}_j(t) &= 6\bm{a}_{3,j}\,t   + 2\bm{a}_{2,j}, \\
    \end{aligned}
\end{equation}
by piecing them together at junctions with Lemma \ref{lma:interiorPt} to determine the constants of integration $\bm{a}_{0,j}$--$\bm{a}_{3,j}$ for each arc $j$.
Note that the only nonlinearities are with respect to $t$, and thus, we can rewrite the system of equations in the following form at each junction $t_k$,
\begin{equation} \label{eq:SOE}
    \begin{aligned}
        \bm{A}(t_k)\, \bm{c} &= \bm{z}, \\
        \Big(\dot{\bm{u}}(\bm{c}, t_k^-) - \dot{\bm{u}}(\bm{c}, t_k^+)\Big)\cdot \bm{v}(t_k) &= 0
\end{aligned}
\end{equation}
where $\bm{A}(t_k)$ is a $16\times16$ matrix containing powers of $t_k$ and constants, $\bm{c}$ is a $16\times1$ matrix of unknown trajectory coefficients, and $\bm{z}$ contains zeros and vertex positions.
Our next result proves the linear equations of \eqref{eq:SOE} are invertible when the values of $t_k$ are distinct, which yields the coefficients $\bm{c}$ as explicit values of the $t_k$'s; this enables us to converted Problem \ref{prb:ocp} into an equivalent shortest path problem.

\begin{lemma} \label{lma:invertible}
    If the junction times satisfy $t^0 < t_1 < t_2 \dots < t^f$, then the matrix $\bm{A}$ in \eqref{eq:SOE} is invertible.
\end{lemma}

\begin{proof}
    Under Lemma \ref{lma:interiorPt}, matrix $\bm{A}$ is constructed using the polynomial powers of $t^0, t_1, t_2, \dots, t^f$ from \eqref{eq:MP}.
    Each row corresponds to a boundary condition or continuity equation, which are linearly independent for distinct values of $t^0, t_1, \dots, t^f$.
    Thus, $\bm{A}$ is invertible by the invertible matrix theorem.
\end{proof}


\begin{problem} \label{prb:graph}
Generate the finite sequence of vertices $\mathcal{S} = \{s_n\} \subset \{1, 2, \dots, N\}$ that yields the minimum-energy trajectory,
\begin{align*}
    &\min_{\mathcal{S}}  ||\bm{u}||^2 \\
    \text{subject to:}&  \\
    &\bm{x}(t^0) = \bm{x}^0, \, \bm{x}(t^f) = \bm{x}^f, \\
    &\eqref{eq:MP}, \quad j=0,1,2,\dots,|\mathcal{S}| \\
    &R - d_k\big(P_k(\bm{p})\big) \leq 0, \quad \forall k\in\mathcal{F}, \\
    &\eqref{eq:lmaFeasible} \text{ or } \eqref{eq:stop}, \, \eqref{eq:SOE} \quad k= 1, 2, \dots, |\mathcal{S}|,
\end{align*}
where $|\mathcal{S}|$ denotes the length of sequence $\mathcal{S}$, which is not known a priori.
\end{problem}

The solution space is a fully connected graph, where the edges correspond to unconstrained trajectory segments that connect boundary conditions and constraint junctions (vertices).
Unfortunately Problem \ref{prb:graph} is a boundary value problem that is not readily amenable to a traditional shortest path formulation, e.g,. Dijkstra's algorithm or A*.
The cost between any pair of nodes depends on the arrival times at those nodes, which must be calculated for an entire trajectory using \eqref{eq:SOE}.
Thus, we propose a distance-informed optimal trajectory generation algorithm using path prefixes, which we define next.

\begin{definition} \label{def:prefix}
    A sequence $\mathcal{S} = \{s_n \in \mathcal{F}\}_{n=1,\dots,N}$ of length $N$ is a \emph{feasible prefix} 
    if the corresponding trajectory satisfies,
    \begin{equation}
       R - d_k\big( P_k(\bm{p}) \big) \leq 0,
    \end{equation}
    for all $k\in\mathcal{F}$ and for all $t\in[t^0, t_{N}], t_{N} < t^f$.
    The distance of the prefix is given by,
    \begin{equation}
        ||\mathcal{S}|| = \sum_{i=2}^k ||\bm{c}_i - \bm{c}_{i-1}|| + ||\bm{c}_1 - \bm{p}_0 || + ||\bm{p}_f - \bm{c}_k ||,
    \end{equation}
    which captures the straight-line distance between the initial state and final state while passing through each vertex.
\end{definition}

We propose a distance-informed search of the fully-connected graph to find the minimum distance sequence that yields a feasible trajectory in Algorithn \ref{alg:main}.
Finding the minimum-distance trajectory is a useful approximation, as the unconstrained optimal trajectory is a straight line.
In this sense, the minimum-distance trajectory corresponds to the minimum deviation from the unconstrained optimal trajectory.

\begin{algorithm2e}
\caption{Our algorithm to generate the minimum-distance feasible trajectory.}\label{alg:main}
\KwData{$\mathcal{V}, \mathcal{F}$ (boundary conditions, vertices, indices)}
\KwResult{$\mathcal{S}$ (minimum distance sequence)}
$\mathcal{P} \gets \{k\}, \, k\in \mathcal{F}$ (initialize the set of prefixes)\;
$d_s = \infty$ (shortest distance found so far)\;
\While{$|\mathcal{P}| \neq 0$ and $\min_{\mathcal{Q}\in\mathcal{P}}\{||\mathcal{Q}||\} < d_s$ }{
    $\mathcal{S}' = \arg\min_{\mathcal{Q} \in \mathcal{P}} \{ ||\mathcal{Q}|| \}$\;
    $\mathcal{P} = \mathcal{P} \setminus \mathcal{S}'$\;
    \If{$\mathcal{S}'$ is feasible for Problem \ref{prb:graph}} {
        $\mathcal{S} = \mathcal{S}'$\;
        $d_s = ||\mathcal{S}||$\;
    }
    \For{$k\in\mathcal{F}\setminus\mathcal{S}'$} {
        $\mathcal{Q} = \mathcal{S}'\cup \{k\}$\;
        Solve \eqref{eq:SOE} for $\mathcal{Q}$\;
        \If{$\mathcal{Q}$ is a feasible prefix} {
            $\mathcal{P} \gets \mathcal{Q}$\;
        }
    }
}
\end{algorithm2e}

\begin{property} \label{prp:bound}
Eq. \eqref{eq:SOE} is a system of equations that yields a lower bound on the energy cost for any trajectory that contains the vertices $\mathcal{S}$ as a sub-sequence.
\end{property}

\begin{proof}
    The system of equations \eqref{eq:SOE} yield the sequence of times $t_1, t_2, \dots$ and trajectory coefficients that are energy-optimal but not necessarily feasible.
    By definition of optimality, no trajectory containing $\mathcal{S}$ as a sub-sequence can cost less energy than $\mathcal{S}$.
\end{proof}

Note that Algorithm \ref{alg:main} yields a sequence of vertices that minimizes the straight-line distance from the initial to final position.
Then, we use Property \ref{prp:bound} to determine the lower bound for the energy consumption of each remaining prefix $\mathcal{Q}\in\mathcal{P}$, and any prefix with a higher cost than the minimum-distance trajectory is discarded.
The remaining prefixes can be expanded using Algorithm \ref{alg:main} to determine the minimum-energy trajectory.
Algorithm \ref{alg:main} can also be used to generate feasible suffixes, i.e., by working backward from the final state to the origin, which may yield a different sequence.
Unfortunately neither approach is guaranteed to find the globally optimal sequence, which may require a combinatorial search through the entire domain.
This is because feasibility depends on the entire trajectory, and our algorithm neglects trajectories where any prefix subsequence is infeasible.
However, we demonstrate that our approach performs well in simulation, and we present a final result about how our algorithm scales into higher-dimensions. 

\begin{property}
    The complexity of Algorithm \ref{alg:main} is polynomial with respect to the state space of the robot.
\end{property}

\begin{proof}
For a fixed number of vertices, the complexity of Algorithm \ref{alg:main} is affected by solving \eqref{eq:SOE} to determine if a candidate prefix $\mathcal{Q}$ is feasible.
By Lemma \ref{lma:invertible} the matrix $\bm{A}$ is invertible and its inverse be computed offline.
Thus, the system of equations \eqref{eq:SOE} requires a matrix multiplication 
and root-finding, both of which have a polynomial computational complexity 
\cite{Borwein1987PiComplexity}.
\end{proof}

\subsection{Efficient Feasibility Check} \label{sec:subOverlap}

A significant challenge for our proposed algorithm is verifying whether a trajectory is feasible.
Here, we propose an algorithm that yields an analytical solution for the time of any constraint violation(s).
First, given a trajectory segment over the interval $[t_1, t_2]$ and polygon edge $k\in\mathcal{F}$ parameterized by $\lambda_k\in[0, 1]$, we seek a solution to the system of equations,
\begin{align}
    \bm{a}_3 t^3 + \bm{a}_2 t^2 + \bm{a}_1 t + \bm{a}_0 = \lambda_k \bm{c}_{k_1} + (1-\lambda_k) \bm{c}_{k_2},
\end{align}
which consists of two equations and two unknowns.
Solving the right hand side for $\lambda_k$ yields,
\begin{align}
    \lambda_k = \frac{1}{c_{k_2}^{x} - c_{k_1}^{x}} \Big(  a_3^x t^3 + a_2^x t^2 + a_1^x t + a_0^x - c_{k_2}^{x} ), \\
    \lambda_k = \frac{1}{c_{k_2}^{y} - c_{k_1}^{y}} \Big(  a_3^y t^3 + a_2^y t^2 + a_1^y t + a_0^y - c_{k_2}^{y} ),
\end{align}
where the superscripts $x, y$ denote the respective vector components.
Setting both equations equal and combining terms yields a cubic polynomial for the crossing time $t$,
\begin{align}
    t^3   \Big\{ a_3^x(c_{k_2}^{y} - c_{k_1}^{y}) -  a_3^y(c_{k_2}^{x} - c_{k_1}^{x}) \Big\} \notag\\
    + t^2 \Big\{ a_2^x(c_{k_2}^{y} - c_{k_1}^{y}) -  a_2^y(c_{k_2}^{x} - c_{k_1}^{x}) \Big\} \notag\\
    + t   \Big\{ a_1^x(c_{k_2}^{y} - c_{k_1}^{y}) -  a_1^y(c_{k_2}^{x} - c_{k_1}^{x}) \Big\} \notag\\
    + \Big\{ (a_0^x - c_{k_2}^{x})(c_{k_2}^{y} - c_{k_1}^{y}) \notag\\
    - (a_0^y - c_{k_2}^{y})(c_{k_2}^{x} - c_{k_1}^{x}) \Big\} = 0, \label{eq:tCubicCrossing}
\end{align}
which has a straightforward analytical solution yielding three values of $t$.
Any imaginary roots of \eqref{eq:tCubicCrossing}, and real roots satisfying $t\not\in[t_1, t_2]$, can be discarded, as they do not correspond to a constraint violation.
For each $t\in[t_1, t_2]$, a constraint violation may only occur when $\lambda_k \in[0, 1]$;
if any $\lambda_k \in (0, 1)$ the trajectory crosses edge $k$ and is infeasible under Assumption \ref{smp:environment}, whereas if $\lambda_k = 0$ or $\lambda_k = 1$, then Lemma \ref{lma:feasible} or Corollary \ref{cor:stop} determine whether the trajectory is feasible for convex and non-convex corners, respectively.

\section{Simulation Results} \label{sec:sim}

We validated the performance of our algorithm using two approaches.
First, we randomly generated polygonal environments with the following procedure: 1) uniformly place $50$ points in a $10\times 10$ m domain at random and remove any points within $1$ m of the initial and final position, 2) cluster the points into $12$ groups using $k$-nearest neighbors, 3) discard any clusters with one or two points, and take the convex hull of the remaining clusters to generate obstacles.
In addition to this statistical study, we also generated trajectories in MATLAB's \emph{complex map} environment and compared the results with the built-in RRT* and Probabilistic Road Map (PRM) implementations with the parameters suggested by the documentation.

\begin{figure}[ht]
    \centering
    \includegraphics[width=0.8\linewidth]{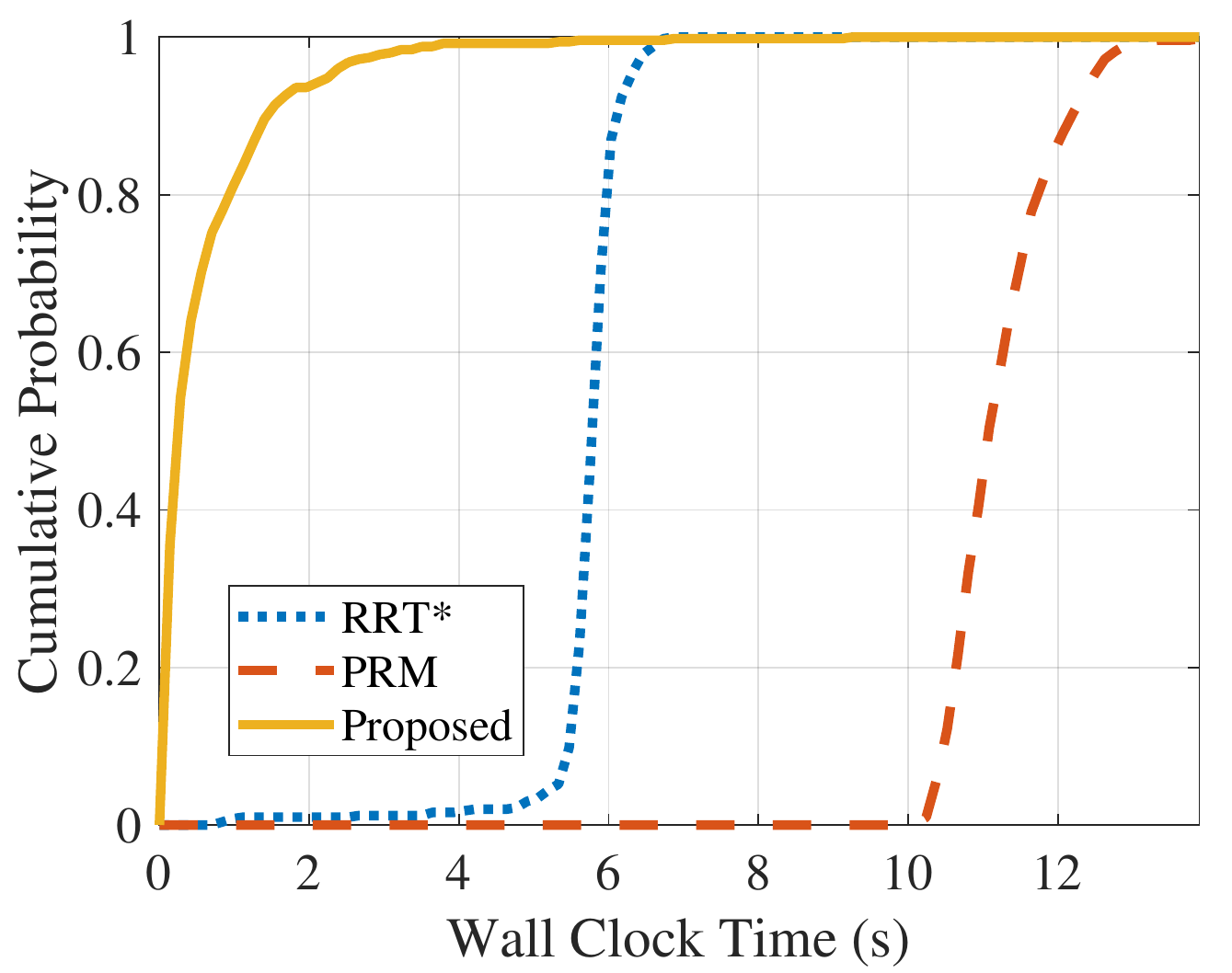}
\caption{Wall clock time required to generate the minimum-distance optimal trajectories using Algorithm \ref{alg:main}, RRT*, and PRM for $500$ randomly generated environments. }
    \label{fig:statistical-results}
\end{figure}

Figure \ref{fig:statistical-results} shows a cumulative distribution function for the wall-clock time for $500$ randomized trials using our proposed approach, RRT*, and PRM.
In these unstructured environments, our approach out-performs RT* and PRM by an order of magnitude.
Our proposed approach takes less than a second to resolve in over 90\% of cases, whereas fewer than 10\% of RRT* and PRM trajectories take less than $5$ seconds and $10$ seconds, respectively.
We generated the RRT* and PRM solutions using $2500$ nodes, as this yielded paths of a similar length to our solution.

We note that given a sequence of vertices, our approach takes approximately $10$ ms to find the optimal sequence $t_k$.
The remaining computational effort checking constraint satisfaction and manipulating the queue by adding and removing prefixes, which are computationally tractable for unstructured environments of disconnected polygons.
Fig. \ref{fig:environments} demonstrates the resulting trajectories for environments with short, medium, and long computational times.

\begin{figure*}
    \centering
    \includegraphics[width=.3\linewidth]{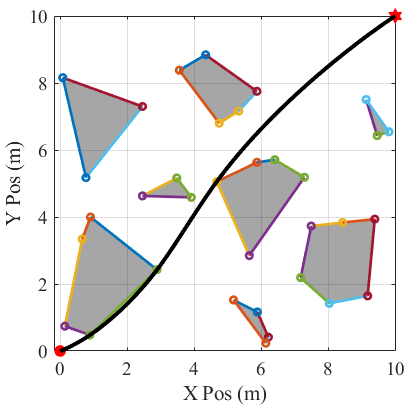}
    \includegraphics[width=.3\linewidth]{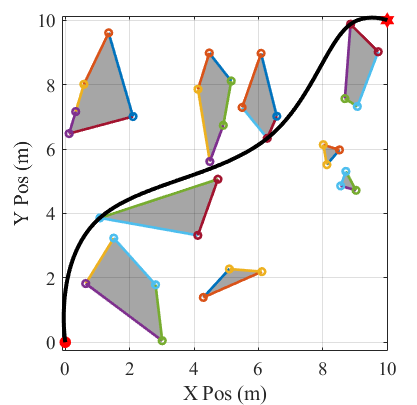}
    \includegraphics[width=.3\linewidth]{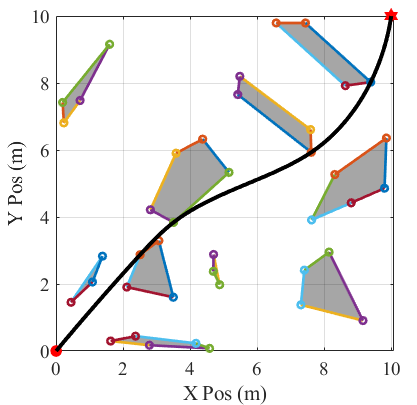}
    \caption{From left to right, the environments and trajectories for short ($t=0.14$ s), medium ($t=1.9$ s), and long ($t=3.4$ s) computational times.}
    \label{fig:environments}
\end{figure*}

Next, we compared our proposed algorithm to RRT* and PRM using MATLAB's built-in \emph{complex environment} map.
We also emphasize that both of these sampling-based methods can only generate the minimum-distance path, and they have no concept of energy consumption.
To generate a minimum-energy path, the methods would need to sample the five-dimensional space $(\bm{p}, \bm{v}, t)$ instead of the two-dimensional position space, and as a rule of thumb, the complexity of sampling methods grows exponentially with the dimension of the state space.
Furthermore, the value of $t$ along a particular path must increase monotonically from $t^0$ to $t^f$, and na\"{i}vely sampling $t$ from a uniform distribution is insufficient to achieve this.

The performance of each algorithm as a function of wall-clock time is presented in Figs. \ref{fig:costs} and \ref{fig:distance}.
Fig. \ref{fig:costs} presents the lower bound on the path cost for RRT* and PRM; this demonstrates that our approach outperforms the other methods by an order of magnitude.
In fact, the trajectories generated by RRT* and PRM tend to induce small oscillations that increase the path cost significantly.
Fig. \ref{fig:distance} demonstrates that RRT* takes a similar amount to match the path length of our proposed algorithm even though we do not optimize for path length, and the performance of PRM is significantly worse in both regards.
Our approach also requires less memory usage, as we only consider combinations of vertices while RRT* and PRM sample tens of thousands of points.
The bottleneck of Algorithm \ref{alg:main} is manipulating the vertex queues, and thus our approach does not outperform RRT* and PRM as significantly in this dense, structured environment with a large number of vertices and edges.

\begin{figure}[ht]
    \centering
    \includegraphics[width=0.8\linewidth]{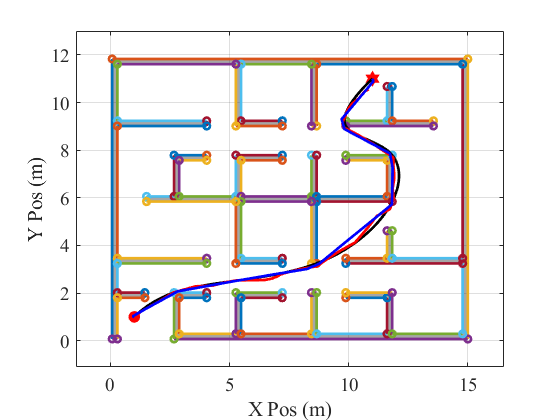}
    \caption{MATLAB's \emph{complex environment} showing the trajectories generated by RRT* (144 s clock time) in red, PRM (155 s clock time) in blue, and our proposed solution (32 s clock time) in black.}
    \label{fig:complex-map}
\end{figure}

\begin{figure}
    \centering
    \includegraphics[width=0.8\linewidth]{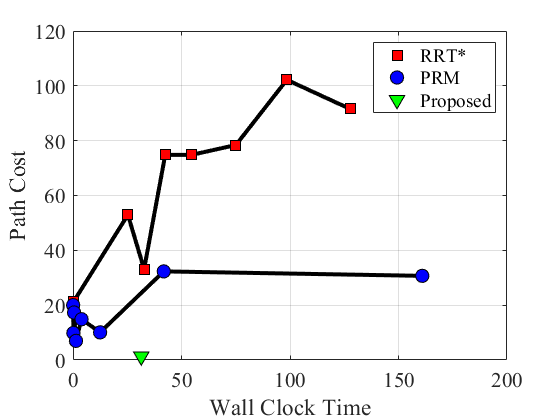}
    \caption{Total cost for different RRT* and PRM iterations; the cost for RRT* and PRM are estimated using \eqref{eq:SOE}, which constitutes a lower bound for each method.}
    \label{fig:costs}
\end{figure}

\begin{figure}
    \centering
    \includegraphics[width=0.8\linewidth]{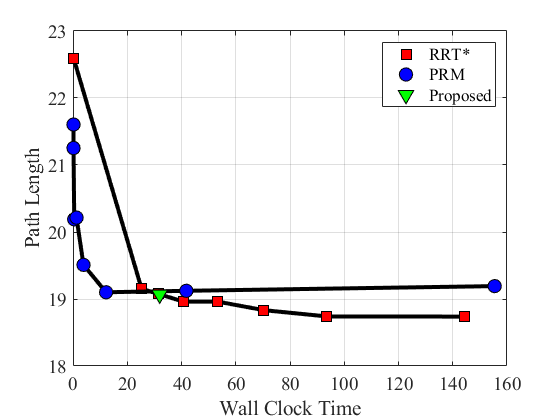}
    \caption{Total path length for different numbers of RRT* and PRM iterations; our proposed approach is close to the performance of both methods despite the fact that we do not minimize path length.}
    \label{fig:distance}
\end{figure}

\section{Experimental Results} \label{sec:experiment}

Finally, we demonstrate that our generated trajectory satisfies Assumption \ref{smp:tracking} by performing an experiment with a BitCraze CrazyFlie quadrotor.
To track the trajectory we employed the CrazySwarm software package of \cite{crazyswarm}, which has built-in support for tracking polynomial position trajectories that are up to 8th order.
The environment and generated trajectory are presented in Fig. \ref{fig:experiment}; it is an $2.5\times 5$ m arena with $0.6$ m ($2$ ft) square virtual obstacles that we inflated by $10$ cm to account for the size of the Crazyflie and tracking errors.
We placed the virtual obstacles in an alternating grid to create a series of connected corridors, and we generated the trajectory offline as a sequence of cubic splines.
We captured the quadrotor's position data at $100$ Hz using a VICON motion capture system, and trajectory tracking was handled using the default paramters in Crazyswarm.
To synchronize the start time of the experiment with the desired trajectory, we selected the initial VICON frame that minimized the cumulative root-mean-square error of the drone's trajectory.

The tracking error at each time instant is presented in Fig. \ref{fig:error}, which measures the Euclidean distance between the desired and actual position of the drone at each time instant.
We observed oscillations in the Crazyflie's position before and after the experiment.
These oscillations are a likely cause of the observed tracking error, which grows and oscillates with no apparent dependence on the quadrotor's speed.
Despite this, the quadrotor tracked the reference trajectory with a maximum deviation of $3$ cm throughout.
Given the CrazyFlie's radius of $6.5$ cm, this deviation was sufficiently small to avoid all of the virtual obstacles.

\begin{figure}[ht]
    \centering
    \includegraphics[width=\linewidth]{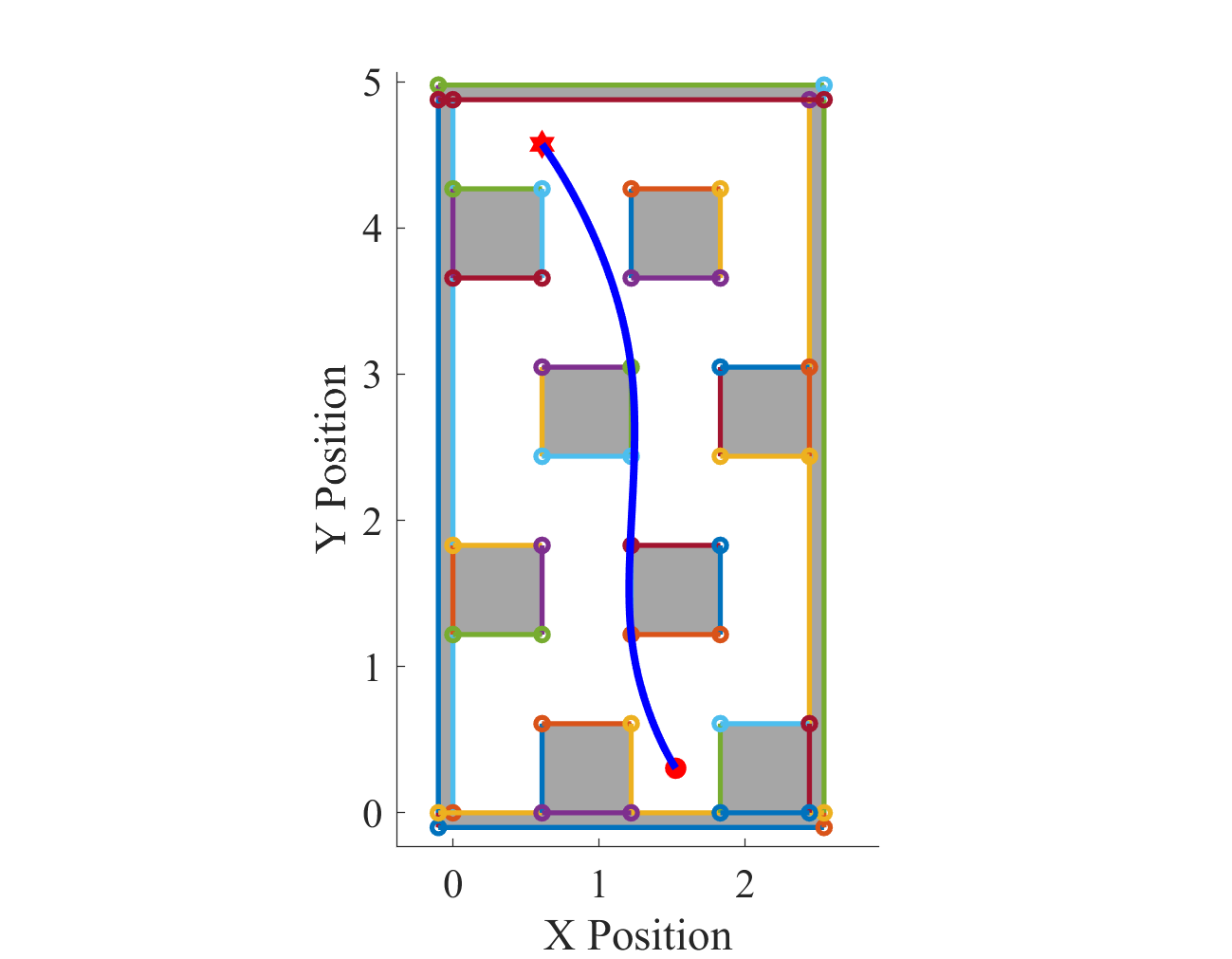}
    \caption{Experimental setup with inflated square obstacles. The robot starts at the (bottom) red circle and navigates to the (top) red star.}
    \label{fig:experiment}
\end{figure}

\begin{figure}
    \centering
    \includegraphics[width=0.8\linewidth]{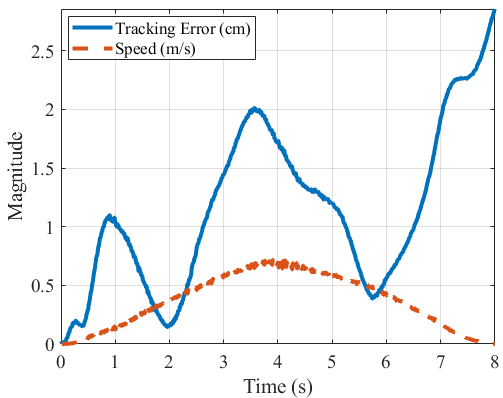}
    \caption{Speed and instantaneous tracking error over the duration of the experiment.}
    \label{fig:error}
\end{figure}

\section{Conclusion} \label{sec:conclusion}

In this article, we proposed a minimum-energy formulation for an autonomous robot to navigate an environment filled with polygonal obstacles.
We transformed the optimal control problem, which minimizes over a function space, into a sequencing problem, which minimizes over a finite set of vertices.
Furthermore, we demonstrated that our approach has polynomial scaling with respect to the size of the robot's state space, and our approach finds shorter and more energy-efficient trajectories significantly faster than comparable sampling-based methods.
Finally, we demonstrated that our reference trajectory can be tracked by a physical robot.

Relaxing Assumptions \ref{smp:bounds} and \ref{smp:environment} to solve the more general optimal control problem is a clear research direction, along with further physical experiments where robots must localize and perceive the obstacles.
Formulating a receding horizon control problem is another compelling research direction, where the robot must re-generate its energy-optimal trajectory as it explores the environment.
Including non-spatial constraints, such as reduced speed zones, to augment our proposed graph-based approach is another interesting research direction.
Finally, determining under what conditions a suffix or prefix search yields the energy-optimal vertex sequence is necessary to guarantee solution optimality; our prefix-based approach is somewhat similar to RRT, so introducing sampling could help fully explore the feasible space.

\bibliography{mendeley,IDS_Pubs}

\end{document}